\documentclass[runningheads,a4paper]{llncs}

\usepackage{amssymb}
\usepackage{graphicx}
\usepackage{epstopdf}

\usepackage{url}
\urldef{\mailsa}\path|zhengs123@gmail.com|
\urldef{\mailsb}\path|chqding@uta.edu|
\newcommand{\keywords}[1]{\par\addvspace\baselineskip
\noindent\keywordname\enspace\ignorespaces#1}

\usepackage{amsmath}

\usepackage[abs]{overpic}
\usepackage{caption}
\usepackage{subcaption}
\usepackage{algorithm}
\usepackage{hyperref}
\usepackage{algpseudocode}

\def \XX {X}

\def \DD {D}
\def \GG {G}
\def \II {I}
\def \AA {A}
\def \BB {B}

\def \xx {{\bf x}}

\def \mm {{\bf m}}

\def \SS {S}
\def\half{{\frac{1}{2}}}
\def \YY {Y}
\def \YYt {{\widetilde{Y}}}

\def \nt {{\tilde{n}}}
\def\L1 {$L_1$}
\def\KK{{\cal K }}
\def\Tr{{\mbox{Tr}}}
\def \CC {C}
\def \ZZ {Z}
\def\diag{{\mbox{diag}}}

\begin{document}

\mainmatter  % start of an individual contribution

% first the title is needed
\title{Kernel Alignment Inspired \\Linear Discriminant Analysis}

% a short form should be given in case it is too long for the running head
\titlerunning{Kernel Alignment Inspired Linear Discriminant Analysis}

\author{Shuai Zheng and Chris Ding}
\authorrunning{S. Zheng and C. Ding}
\institute{Department of Computer Science and Engineering,\\
University of Texas at Arlington, TX, USA\\
\mailsa, \mailsb
}

\toctitle{Kernel Alignment Inspired Linear Discriminant Analysis}
\tocauthor{Shuai Zheng and Chris Ding}
\maketitle
\begin{abstract}
Kernel alignment measures the degree of similarity between two kernels. In
this paper, inspired from kernel alignment, we propose a new Linear
Discriminant Analysis (LDA) formulation, kernel alignment LDA (kaLDA). We
first define two kernels, data kernel and class indicator kernel. The problem
is to find a subspace to maximize the alignment between subspace-transformed
data kernel and class indicator kernel. Surprisingly, the kernel alignment
induced kaLDA objective function is very similar to classical LDA and can be
expressed using between-class and total scatter matrices. This can be
extended to multi-label data. We use a Stiefel-manifold  gradient descent
algorithm to solve this problem. We perform experiments on 8 single-label and
6 multi-label data sets. Results show that kaLDA has very good performance on
many single-label and multi-label problems.
\keywords{Kernel Alignment, LDA}
\end{abstract}

\section{Introduction}

Kernel alignment~\cite{kernel2001} is a way to incorporate class label
information into kernels which are traditionally directly constructed from
data without using class labels. Kernel alignment can be viewed as a
measurement of consistency between the similarity function (the kernel) and
class structure in the data. Improving this consistency helps to enforce data
become more separated when using the class label aligned kernel. Kernel
alignment has been applied to pattern recognition and feature selection
recently~\cite{cristianini2007method,zhu2004nonparametric,hoi2006learning,howard2009transformation,cuturi2011fast}.

In this paper, we find that if we use the widely used linear kernel and a kernel built from class indicators,
the resulting kernel alignment function is very similar to the widely used linear discriminant analysis (LDA),
using the well-known between-class scatter matrix $\SS_b$ and total scatter matrix $\SS_t$.
We call this objective function as kernel alignment induced LDA (kaLDA).
If we transform data into a linear subspace, the optimal solution is to maximize this kaLDA.

We further analyze this kaLDA and propose a Stiefel-manifold gradient descent
algorithm to solve it. We also extend kaLDA to multi-label problems.
Surprisingly, the scatter matrices arising in multi-label kernel alignment
are identical those matrices developed in Multi-label
LDA~\cite{wang2010multi}.

We perform extensive experiments by comparing kaLDA with other approaches on 8 single-label datasets and 6 multi-label data sets.
Results show that kernel alignment LDA approach has good performance in terms of classification accuracy and F1 score.

\section{From Kernel Alignment to LDA}
Kernel Alignment is a similarity measurement between a kernel function and a target function. In other words, kernel alignment evaluates the degree of fitness between the data in kernel space and the target function. For this reason, we usually set the target function to be the class indicator function. The other kernel function is the data matrix. By measuring the similarity between data kernel and class indicator kernel, we can get a sense of how easily this data can be separated in kernel subspace. The alignment of two kernels $\KK_1$ and $\KK_2$ is given as \cite{kernel2001}:
\begin{align}
A(\KK_1, \KK_2)= \frac{\Tr(\KK_1\KK_2)}{\sqrt{\Tr(\KK_1\KK_1)}\sqrt{\Tr(\KK_2\KK_2)}}. \label{eq:al}
\end{align}

We first introduce some notations, and then present Theorem \ref{th:th1} and kernel alignment projective function.

Let data matrix be $\XX \in \Re^{p \times n}$ and $\XX = (\xx_1, \cdots, \xx_n)$, where $p$ is data dimension, $n$ is number of data points, $\xx_i$ is a data point. Let normalized class indicator matrix be $\YY \in \Re^{n \times K}$, which was used to prove the equivalence between PCA and K-means clustering
\cite{Zha:kmeans,ding2004k}, and
\begin{align}
\label{eq:y}
\YY_{ik} =
\begin{cases}
\frac{1}{\sqrt{n_k}}, & \mbox{if point } i\mbox{ is in class } k. \\
0, & \mbox{otherwise}.
\end{cases}
\end{align}%
where $K$ is total class number, $n_k$ is the number of data points in class $k$. Class mean is $\mm_k=\sum_{\xx_i \in k}\xx_i/n_k$ and total mean of data is $\mm =\sum_i\xx_i/n$.

\begin{theorem}
\label{th:th1}
Define data kernel $\KK_1$ and class label kernel $\KK_2$ as follows:
\begin{align}
\KK_1=\XX^T\XX, \;\; \KK_2=\YY\YY^T,
\end{align}%
we have
\begin{align}
A(\KK_1,\KK_2)= c\frac{\Tr\SS_b}{\sqrt{\Tr\SS_t^2}} \label{eq:ak}
\end{align}%
where $c = 1/\sqrt{\Tr(\YY\YY^T)^2}$ is a constant independent of $\XX$.

Furthermore, let $\GG \in \Re^{p \times k}$ be a linear transformation to a $k$-dimensional subspace
\begin{align}
\widetilde \XX = \GG^T \XX, \;\; \widetilde\KK_1=\widetilde\XX^T\widetilde\XX,
\end{align}%
we have%
 \begin{align}
A(\widetilde\KK_1,\KK_2)= c\frac{\Tr(\GG^T \SS_b \GG)}{\sqrt{\Tr(\GG^T\SS_t\GG)^2}} \label{eq:akt}
 \end{align}%
where%
\begin{align}
\SS_b&=\sum_{k=1}^K n_k (\mm_k - \mm) (\mm_k - \mm)^T,\label{eq:sb0}\\
\SS_t&=\sum_{i=1}^n (\xx_i-\mm) (\xx_i-\mm)^T,
 \end{align}
\end{theorem}%
Theorem \ref{th:th1} shows that kernel alignment can be expressed using scatter matrices $\SS_b$ and $\SS_t$.
In applications, we adjust $\GG$ such that kernel alignment is maximized, i.e., we solve the following problem:
\begin{align}
&\max_{\GG}  \frac{\Tr(\GG^T \SS_b \GG)}{\sqrt{\Tr(\GG^T\SS_t\GG)^2}}.
\label{eq:j10}
\end{align}%
In general, columns of $G$ are assumed to be linearly independent.

A striking feature of this kernel alignment problem is that it is very similar to classic LDA.

\subsection{Proof of Theorem 1 and Analysis}

Here we note a useful lemma and then prove Theorem \ref{th:th1}.

In most data analysis, data are centered, i.e., $\sum_i \xx_i = {\bf 0}$.
Here we assume data is already centered. The following results remain correct
if data is not centered.
We have the following relations:%
\begin{lemma}
\label{lm:lm1}
Scatter matrices $\SS_b, \SS_t$ can be expressed as:
\begin{align}
\SS_b=&\XX\YY\YY^T\XX^T, \label{eq:sb1}\\
\SS_t=&\XX\XX^T. \label{eq:st1}
\end{align}
\end{lemma}%
These results are previously known, for example, Theorem 3 of
\cite{ding2004k}.

{\bf Proof of Theorem \ref{th:th1}}.
To prove Eq.(\ref{eq:ak}), we substitute $\KK_1, \KK_2$ into Eq.(\ref{eq:al}) and obtain, noting $\Tr(\AA\BB) = \Tr(\BB\AA)$,
\begin{align}
A(\KK_1,\KK_2) &= \frac{\Tr(\XX\YY\YY^T \XX^T )}{\sqrt{\Tr(\XX\XX^T)^2}\sqrt{\Tr(\YY\YY^T)^2}}
=c\frac{\Tr\SS_b}{\sqrt{\Tr\SS_t^2}}.  \nonumber
\end{align}%
where we used Lemma \ref{lm:lm1}. $c = 1/\sqrt{\Tr(\YY\YY^T)^2}$ is a constant independent of data $\XX$.

To prove Eq.(\ref{eq:akt}),
\begin{align}
&A(\widetilde\KK_1,\KK_2)
=c\frac{\Tr( \GG^T\XX\YY\YY^T \XX^T \GG ) }{\sqrt{\Tr(\GG^T\XX\XX^T\GG)^2}}
= c\frac{\Tr(\GG^T \SS_b \GG)}{\sqrt{\Tr(\GG^T\SS_t\GG)^2}}, \nonumber
\end{align}%
thus we obtain Eq.(\ref{eq:akt}) using Lemma \ref{lm:lm1}.

\subsection{Relation to Classical LDA}

In classical LDA, the between-class scatter matrix $\SS_b$ is defined as Eq.(\ref{eq:sb0}), and
the within-class scatter matrix $\SS_w$ and total scatter matrix $\SS_t$ are defined as:
\begin{align}
\SS_w=\sum_{k=1}^K \sum_{\xx_i \in k}(\xx_i-\mm_k)(\xx_i-\mm_k)^T,\;
\SS_t=\SS_b+\SS_w,
\end{align}%
where $\mm_k$ and $\mm$ are class means. Classical LDA finds a projection matrix $\GG \in \Re^{p \times (K-1)}$ that
minimizes $\SS_w$ and maximizes $\SS_b$ using
 the following objective:
\begin{align}
\label{eq:jlda1}
\max_{\GG}  \Tr \frac{\GG^T\SS_b\GG}{\GG^T\SS_w\GG},
\end{align}%
or%
\begin{align}
\label{eq:jlda2}
\max_{\GG} \frac{\Tr(\GG^T\SS_b\GG)}{\Tr(\GG^T\SS_w\GG)}.
\end{align}%
Eq.(\ref{eq:jlda2}) is also called trace ratio (TR) problem \cite{wang2007trace}.
It is easy to see~\footnote{
Eq.(\ref{eq:jlda2}) is equivalent to
$\min \frac{\Tr(\GG^T\SS_w\GG)}{\Tr(\GG^T\SS_b\GG)}$,
 which is $\min \left( \frac{\Tr(\GG^T\SS_w\GG)}{\Tr(\GG^T\SS_b\GG)} + 1 \right)$.
Reversing to maximization and using $\SS_t=\SS_b+\SS_w$, we obtain Eq.(\ref{eq:jlda3}).
}
that Eq.(\ref{eq:jlda2}) can be expressed as
\begin{align}
\label{eq:jlda3}
\max_{\GG} \frac{\Tr(\GG^T\SS_b\GG)}{\Tr(\GG^T\SS_t\GG)}.
\end{align}%
As we can see, kernel alignment LDA objective function Eq.(\ref{eq:j10}) is very similar to Eq.(\ref{eq:jlda3}). Thus kernel alignment provides an interesting alternative explanation of LDA.
In fact, we can similarly show that
in Eq.(\ref{eq:j10}), $\SS_w$ is also maximized as in the standard LDA.
First, Eq.(\ref{eq:j10}) is equivalent to
$$
\max_G \Tr(G^T S_b G) \;\; s.t. \;\; \Tr(G^T S_t G)^2 = \eta,
$$
where $\eta$ is a fixed-value. The precise value of $\eta$ is unimportant, since the scale of $G$ is undefined in LDA: if $G^*$ is an optimal solution,
and $r$ is any real number,
$G^{**} = r G^*$ is also an optimal solution with the same optimal objective function value.
The above optimization is approximately equivalent to
$$
\max_G \Tr(G^T S_b G) \;\;s.t. \;\; \Tr (G^T S_t G) = \eta,
$$
This is same as
$$ \max_G \Tr(G^T S_b G) \;\;  s.t. \;\; \Tr (G^T S_w G) = \eta -  \Tr(G^T S_b G),
$$
In other words, $S_b$ is maximized while $S_w$ is minimized --- recovering the LDA main theme.

\section{Computational Algorithm}

In this section, we develop efficient algorithm to solve kaLDA objective function Eq.(\ref{eq:j10}):
\begin{align}
&\max_{\GG} J_1 = \frac{\Tr(\GG^T \SS_b \GG)}{\sqrt{\Tr(\GG^T\SS_t\GG)^2}}  , \;\; s.t. \;\; \GG^T \GG = \II.
\label{eq:j1}
\end{align}%
The condition $\GG^T \GG = \II$ ensures different columns of $\GG$ mutually independent.
The gradient of $J_1(\GG)$ is%
\begin{align}
\nabla J_1  \triangleq  \frac{\partial{J_1}}{\partial{\GG}}
=  2 \frac{\AA}{\sqrt{\Tr \DD^2}} - 2 \frac{\Tr\BB}{ (\Tr\DD^2)^{\frac{3}{2}}}\CC\DD, \label{eq:dg}
\end{align}%
where $
\AA = \SS_b \GG,\;
\BB = \GG^T \AA, \;
\CC = \SS_t \GG, \;
\DD = \GG^T \CC.$%

Constraint $\GG^T \GG = \II$ enforces $\GG$ on the Stiefel manifold. Variations of $\GG$ on this manifold is parallel transport, which gives some restriction to the gradient. This has been been worked out in~\cite{Edelman98thegeometry}.
The gradient that reserves the manifold structure is
 \begin{align}
 \label{eq:lag1}
 \nabla J_1   - \GG [\nabla J_1]^T \GG.
 \end{align}%
 Thus the algorithm computes the new $\GG$ is given as follows:
 \begin{align}
 \label{eq:lag2}
 \GG \leftarrow  \GG - \eta (\nabla J_1   - \GG [\nabla J_1]^T \GG ).
 \end{align}%
 The step size $\eta$ is usually chosen as:
 \begin{align}
 \label{eq:lag3}
 \eta = \tau \| \GG \|_1/ \|\nabla J_1 - \GG (\nabla J_1)^T \GG \|_1, \; \tau = 0.001 \sim 0.01.
 \end{align}%
 where $\| \GG \|_1 = \sum_{ij} |\GG_{ij}|.$

Occasionally, due to the loss of numerical accuracy, we use projection $\GG \leftarrow \GG (\GG^T \GG)^{-\half}$ to restore $\GG^T \GG = \II$.  Starting with the standard LDA solution of $\GG$,
this algorithm is iterated until the algorithm converges to a local optimal solution. In fact, objective function will converge quickly when choosing $\eta$ properly. Figure \ref{fig:j1} shows that $J_1$ converges in about 200 iterations when $\tau=0.001$, for datasets ATT, Binalpha, Mnist, and Umist (more details about the datasets will be introduced in experiment section). In summary, kernel alignment LDA (kaLDA) procedure is shown in Algorithm \ref{alg:alda}.

\begin{algorithm}[t]
\small
\caption{$[\GG]=kaLDA(\XX, \YY)$}
\label{alg:alda}
\begin{algorithmic}[1]
\Require Data matrix $\XX \in \Re^{p \times n}$, class indicator matrix $\YY \in \Re^{n \times K}$
\Ensure Projection matrix $\GG \in \Re^{p \times k}$
\State Compute $\SS_b$ and $\SS_t$ using Eq.(\ref{eq:sb1}) and Eq.(\ref{eq:st1})
\State Initialize $\GG$ using classical LDA solution
\Repeat
\State Compute gradient using Eq.(\ref{eq:dg})
\State Update $\GG$ using Eq.(\ref{eq:lag2})
\Until $J_1$ Converges
\end{algorithmic}
\end{algorithm}

\begin{figure*}[!h]
\centering
\begin{subfigure}{.5\textwidth}
  \centering
  \includegraphics[width=.9\textwidth]{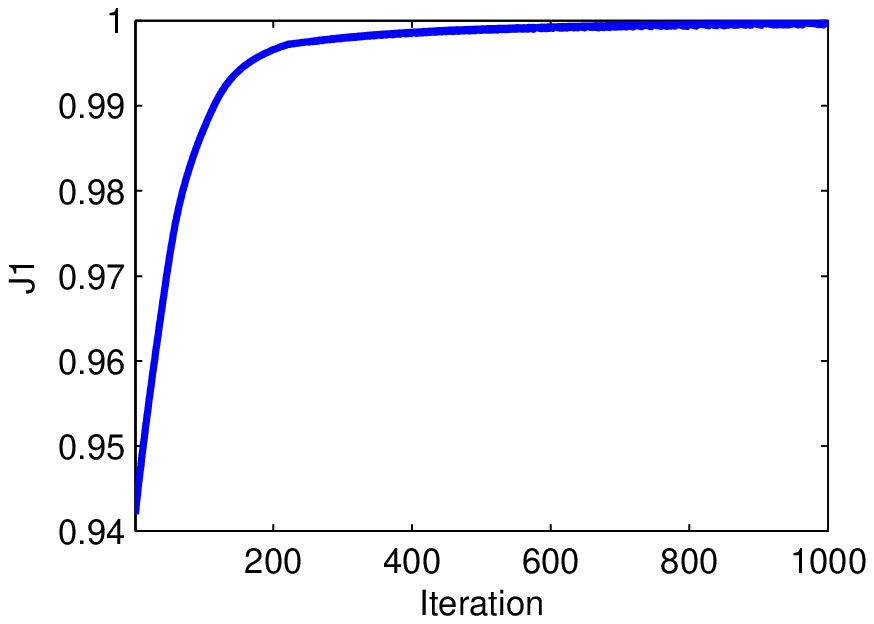}
  \caption{ATT}
  \label{fig:j1_1}
\end{subfigure}%
\begin{subfigure}{.5\textwidth}
  \centering
  \includegraphics[width=.9\textwidth]{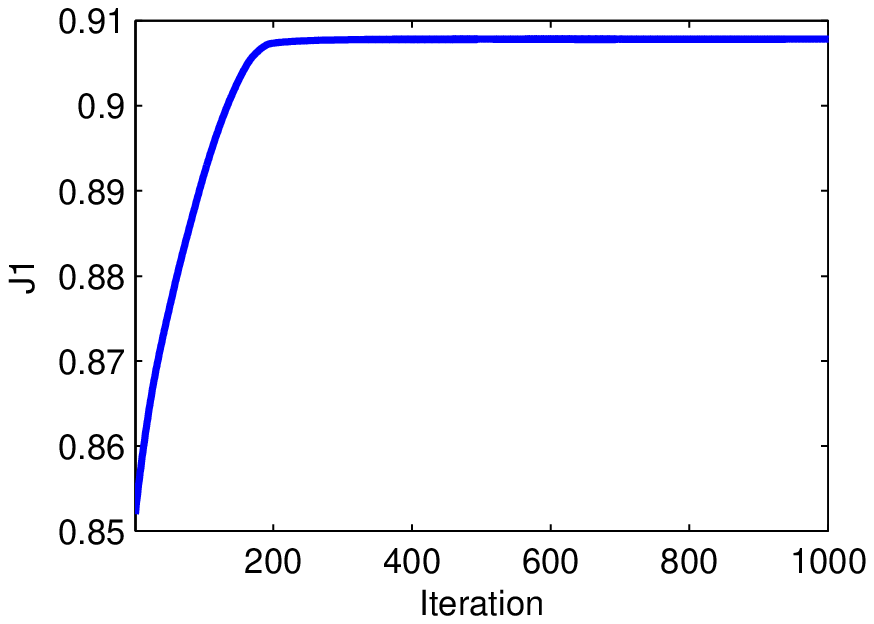}
  \caption{Binalpha}
  \label{fig:j1_2}
\end{subfigure}\\
\begin{subfigure}{.5\textwidth}
  \centering
  \includegraphics[width=.9\textwidth]{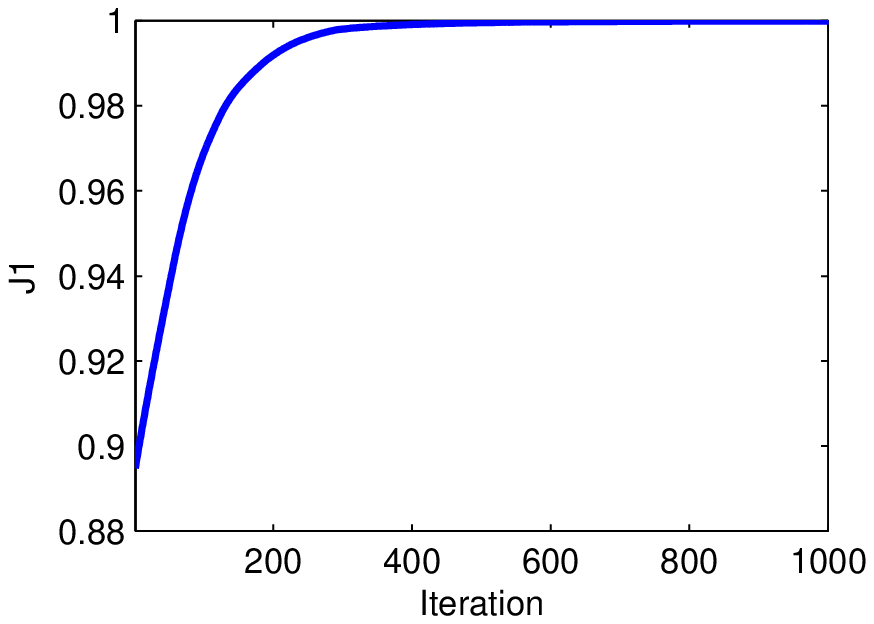}
  \caption{Mnist}
  \label{fig:j1_3}
\end{subfigure}%
\begin{subfigure}{.5\textwidth}
  \centering
  \includegraphics[width=.9\textwidth]{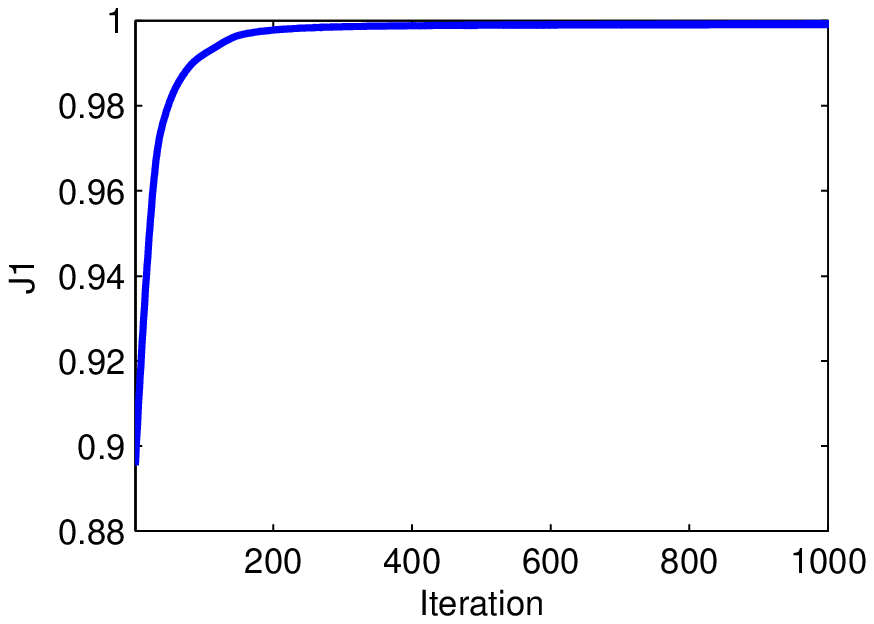}
  \caption{Umist}
  \label{fig:j1_4}
\end{subfigure}
\caption{Objective J1 converges using Stiefel-manifold gradient descent algorithm ($\tau=0.001$).}
\label{fig:j1}
\end{figure*}

To show the effectiveness of proposed kaLDA, we visualize a real dataset in 2-D subspace in Figure \ref{fig:2d}. In this example, we take 3 classes of 644-dimension Umist data, 18 data points in each class. Figure \ref{fig:raw} shows the original data projected in 2-D PCA subspace. Blue points are in class 1; red circle points are in class 2; black square points are in class 3. Data points from the three classes are mixed together in 2-D PCA subspace. It is difficult to find a linear boundary to separate points of different classes. Figure \ref{fig:lda} shows the data in 2-D standard LDA subspace. We can see that data points in different classes have been projected into different clusters. Figure \ref{fig:alda} shows the data projected in 2-D kaLDA subspace. Compared to Figure \ref{fig:lda}, the within-class distance in Figure \ref{fig:alda} is much smaller. The distance between different classes is larger.

\begin{figure*}[!h]
\centering
\begin{subfigure}{.45\textwidth}
  \centering
  \includegraphics[width=\textwidth]{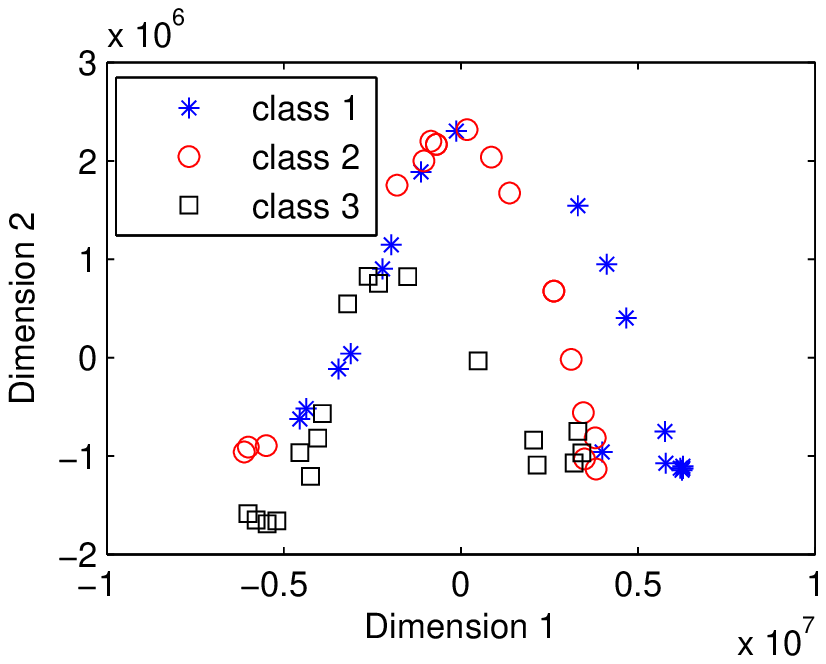}
  \caption{2-D PCA subspace.}
  \label{fig:raw}
\end{subfigure}
\begin{subfigure}{.45\textwidth}
  \centering
  \includegraphics[width=\textwidth]{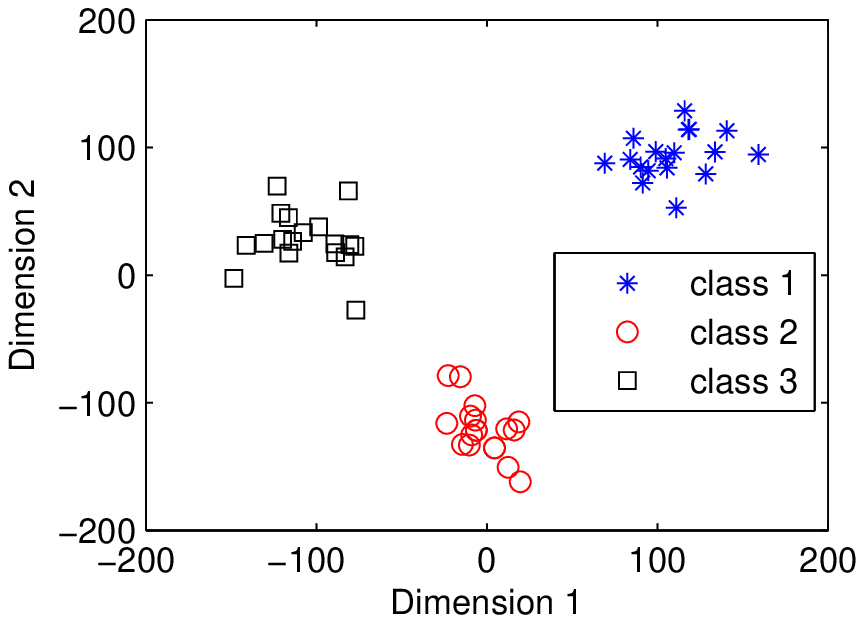}
  \caption{2-D LDA subspace.}
  \label{fig:lda}
\end{subfigure}\\
\begin{subfigure}{.45\textwidth}
  \centering
  \includegraphics[width=\textwidth]{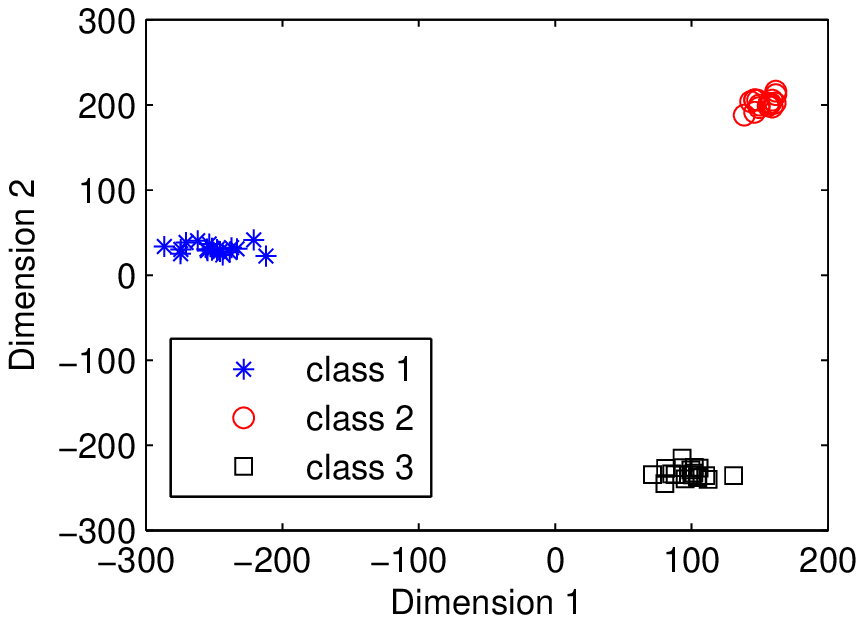}
  \caption{2-D kaLDA subspace.}
  \label{fig:alda}
\end{subfigure}%
\caption{Visualization of Umist data in 2-D PCA, 2-D LDA and 2-D kaLDA subspace.}
\label{fig:2d}
\end{figure*}

\section{Extension to Multi-label Data}

Multi-label problem arises frequently in image and video annotations, multi-topic text categorization, music classification. etc.\cite{wang2010multi}.
In multi-label data,
a data point could have several class labels (belonging to several classes).
For example, an image could have ``cloud", ``building", ``tree" labels.
This is different from the case of single-label problem,
where one point can have only one class label.
Multi-label is very natural and common in our everyday life. For example,
a film can be simultaneously classified as ``drama", ``romance", ``historic" (if it is about a true story).
A news article can have topic labels such as ``economics", ``sports", etc.

Kernel alignment approach can be easily and naturally extended to
multi-label data, because the class label kernel can be clearly and unambiguously defined using class label matrix $\ZZ$ on both single label and multi-label data sets. The data kernel is defined as usual.
In the following we further develop this approach.

One important result of our kernel alignment approach for single label data is that
it has close relationship with LDA.
For multi-label data, each data point could belong to several classes. The standard scatter matrices $\SS_b, \SS_w$ are
ambiguous, because $\SS_b,\SS_w$ are only defined for single label data where each data point belongs to one class only.
However, our kernel alignment approach on multi-label data leads to new definitions of scatter matrices and similar objective function;
this can be viewed as the generalization
of LDA from single-label data to multi-label data via kernel alignment approach.

Indeed, the new scatter matrices we obtained from kernel alignment approach are
identical to the so-called ``multi-label LDA"~\cite{wang2010multi} developed from a class-separate, probabilistic point of view, very different
from our point of view. The fact that these two approaches lead to the same set of scatter matrices show that the resulting multi-label LDA framework
has a broad theoretical basis.

We first present some notations for multi-label data and then describe the kernel alignment approach for multi-label data in Theorem \ref{th:th2}.
The class label matrix $\ZZ \in \Re^{n \times K}$ for data $\XX \in \Re^{p \times n}$ is given as:
\begin{align}
\label{eq:z}
\ZZ_{ik} =
\begin{cases}
1, & \mbox{if point } i\mbox{ is in class } k. \\
0, & \mbox{otherwise}.
\end{cases}
\end{align}%
Let $\nt_k=\sum_{i=1}^n\ZZ_{ik}$ be the number of data points in class $k$. Note that for multi-label data,
$\sum_{k=1}^K \nt_k > n$.
The {\bf normalized} class indicator matrix $\YYt \in \Re^{n \times K}$ is given as:
\begin{align}
\label{eq:y2}
\YYt_{ik} =
\begin{cases}
\frac{1}{\sqrt{\nt_k}}, & \mbox{if point } i\mbox{ is in class } k. \\
0, & \mbox{otherwise}.
\end{cases}
\end{align}%

Let $\rho_i = \sum_{k=1}^K \ZZ_{ik}$ be the number of classes that $\xx_i$
belongs to. Thus $\rho_i$ are the weights of $\xx_i$. Define the diagonal
weight matrix $\Omega = \diag(\rho_1,\cdots, \rho_n)$. The kernel alignment
formulation for multi-label data can be stated as
\begin{theorem}
\label{th:th2}
For multi-label data $\XX$, let the data kernel and class label kernel be
\begin{align}
\KK_1 = \Omega^{\frac{1}{2}} \XX^T \XX \Omega^{\frac{1}{2}}, \;\; \KK_2 = \Omega^{-\frac{1}{2}} \YYt \YYt^T \Omega^{-\frac{1}{2}}.
\end{align}%
We have the alignment
\begin{align}
A(\KK_1,\KK_2)
= c\frac{\Tr\SS_b}{\sqrt{\Tr\SS_t^2}} \label{eq:ak2}
\end{align}%
where $c = 1/\sqrt{\Tr(\Omega^{-1}\YYt \YYt^T)^2}$ is a constant independent of data $\XX$, and
$\SS_b ,\; \SS_t$ are given in Eqs.(\ref{eq:sb3}, \ref{eq:st3}).

Furthermore, let $\GG \in \Re^{p \times k}$ be the linear transformation to a $k$-dimensional subspace,
\begin{align}
\widetilde \XX = \GG^T \XX, \;\; \widetilde\KK_1= \Omega^{1/2}\widetilde\XX^T\widetilde\XX \Omega^{1/2},
 \end{align}%
we have
\begin{align}
A(\widetilde\KK_1,\KK_2)
= c\frac{\Tr(\GG^T \SS_b \GG)}{\sqrt{\Tr(\GG^T\SS_t\GG)^2}} \label{eq:akt2}
\end{align}%
\end{theorem}

The matrices $\SS_b, \SS_t$ in Theorem \ref{th:th2} are defined as:
\begin{align}
&\SS_b = \sum_{k=1}^K \nt_k  (\mm_k-\mm)(\mm_k-\mm)^T,  \label{eq:sb3}\\
&\SS_t= \sum_{k=1}^K \sum_{i=1}^n\ZZ_{ik}(\xx_i-\mm)(\xx_i-\mm)^T, \label{eq:st3}
\end{align}
where  $\mm_k$ is the mean of class $k$ and $\mm$ is global mean, defined as:
\begin{align}
\mm_k = \frac{\sum_{i=1}^n \ZZ_{ik} \xx_i}{\nt_k }, \quad
\mm = \frac{ \sum_{i=1}^n \rho_{i}\xx_i}{ \sum_{k=1}^K \nt_k}.
\end{align}%
Therefore, we can seek an optimal subspace for multi-label data by solving
Eq.(\ref{eq:j1}) with $\SS_b, \SS_t$ given in Eqs.(\ref{eq:sb3},\ref{eq:st3})

\subsection{Proof of Theorem 2 and Equivalence to Multi-label LDA}

Here we note a useful lemma for multi-label data and then prove Theorem \ref{th:th2}.
We consider the case the data is centered, i.e., $\sum_{i=1}^n \rho_i \xx_i ={\bf 0}$.
The results also hold when data is not centered, but the proofs are slightly complicated.

\begin{lemma}
\label{lm:lm2}
For multi-label data, $\SS_b, \SS_t$ of Eqs.(\ref{eq:sb3},\ref{eq:st3}) can be expressed as
\begin{align}
\SS_b=&\XX\YYt \YYt^T\XX^T \label{eq:sb2}\\
\SS_t=&\XX \Omega\XX^T \label{eq:st2}
\end{align}%
\end{lemma}

\begin{proof}
From the definition of $\mm_k$ and $\YYt$ in multi-label data, we have
$$ \XX\YYt = (\mm_1, \cdots, \mm_K)
\begin{pmatrix}
\sqrt{\nt_1} &   & \\
         & \ddots & \\
         &         & \sqrt{\nt_K}\\
\end{pmatrix}.$$
Thus
$ \XX \YYt \YYt^T \XX^T = \sum_{k=1}^K \nt_k \mm_k \mm_k^T$
recovers $\SS_b$ of Eq.(\ref{eq:sb3}).

To prove  Eq.(\ref{eq:st2}), note that
$\XX \Omega = (\rho_1 \xx_1, \cdots, \rho_n \xx_n),$
thus $\XX \Omega \XX^T = \sum_{i=1}^n \rho_i \xx_i \xx_i^T.$

\end{proof}

{\bf Proof of Theorem \ref{th:th2}}.
Using Lemma \ref{lm:lm2}, to prove Eq.(\ref{eq:ak2}),
\begin{align}
&A(\KK_1,\KK_2) = c\frac{\Tr(\XX \YYt \YYt^T\XX^T)}{\sqrt{\Tr(\XX \Omega \XX^T)^2}}
=c\frac{\Tr\SS_b}{\sqrt{\Tr\SS_t^2}},  \nonumber
\end{align}%
where $c = 1/\sqrt{\Tr(\Omega^{-1}\YYt \YYt^T)^2}$ is independent of $\XX$.

To prove Eq.(\ref{eq:akt2}),
\begin{align}
&A(\widetilde\KK_1,\KK_2)=c\frac{\Tr(\GG^T\XX \YYt \YYt^T\XX^T\GG)}{\sqrt{\Tr(\GG^T\XX\Omega\XX^T\GG)^2}} = c\frac{\Tr(\GG^T \SS_b \GG)}{\sqrt{\Tr(\GG^T\SS_t\GG)^2}}. \nonumber
\end{align}

For single-label data, $\rho_i=1, \; \Omega=\II, \; \nt_k = n_k$,  Eqs.(\ref{eq:sb2},\;\ref{eq:st2}) reduce to
Eqs.(\ref{eq:sb1},\;\ref{eq:st1}), and Theorem \ref{th:th2} reduces to Theorem \ref{th:th1}.

As we can see, surprisingly, the scatter matrices $\SS_b, \SS_t$ of Eqs.(\ref{eq:sb3}, \ref{eq:st3}) arising in Theorem \ref{th:th2} are identical to that in Multi-label LDA proposed in~\cite{wang2010multi}.
\begin{table}[t]
\centering
\small
\caption{Single-label datasets attributes.}
\label{tab:data1}
\scalebox{1.2}{
\begin{tabular}{c|ccc}
\hline\hline
Data  & n     & p     & k \\
\hline
Caltec07 & 210	&432&	7\\
Caltec20 & 1230 & 432   & 20 \\
MSRC  & 210   & 432   & 7 \\
ATT	&400&	644	&40\\
Binalpha&	1014&	320	&26\\
Mnist	&150	&784&	10\\
Umist	&360&	644	&20\\
Pie	&680&	1024&	68\\
\hline\hline
\end{tabular}}
\end{table}

\begin{table}[t]
\centering
\caption{Classification accuracy on Single-label datasets ($K-1$ dimension).}
\label{tab:acc}
\resizebox{\columnwidth}{!}{
\begin{tabular}{c|ccccccc}
\hline\hline
Data  & kaLDA  & LDA   & TR    & sdpLDA & MMC   & RLDA   & OCM \\
\hline
Caltec07 & 0.7524 & 0.6619 & 0.6762 & 0.5619 & 0.6000 & \textbf{0.7952}   & 0.7619 \\
Caltec20 & \textbf{0.7068} & 0.6320 & 0.4465 & 0.3386 & 0.5838 & 0.6812 & 0.6696 \\
MSRC  & \textbf{0.7762} & 0.6857 & 0.5714 & 0.5952 & 0.5667 & 0.7333  & 0.7286 \\
ATT   & \textbf{0.9775} & 0.9750 & 0.9675 & 0.9750 & 0.9750 & 0.9675 & 0.9675 \\
Binalpha & 0.7817 & 0.6078 & 0.4620 & 0.2507 & 0.7638 & 0.7983   & \textbf{0.8204} \\
Mnist & \textbf{0.8800} & 0.8733 & 0.8667 & 0.8467 & 0.8467 & 0.8667  & 0.8467 \\
Umist & 0.9900 & 0.9900 & \textbf{0.9917} & 0.9133 & 0.9633 & 0.9800   & 0.9783 \\
Pie   & 0.8765 & \textbf{0.8838} & 0.8441 & 0.8632 & 0.8676 & 0.6515   & 0.6515 \\
\hline\hline
\end{tabular}}
\end{table}

\section{Related Work}

Linear Discriminant Analysis (LDA) is a widely-used dimension reduction and subspace learning algorithm. There are many LDA reformulation publications in recent years. Trace Ratio problem is to find a subspace transformation matrix $\GG$ such that the within-class distance is minimized and the between-class distance is maximized. Formally, Trace Ratio maximizes the ratio of two trace terms, $\max_{\GG}\Tr(\GG^T\SS_b\GG)/\Tr(\GG^T\SS_t\GG)$~ \cite{wang2007trace,jia2009trace}, where $\SS_t$ is total scatter matrix and $\SS_b$ is between-class scatter matrix. Other popular LDA approach includes, regularized LDA(RLDA) \cite{guo2007regularized}, Orthogonal Centroid Method (OCM) \cite{park2003lower}, Uncorrelated LDA(ULDA)~\cite{ye2005characterization}, Orthogonal LDA (OLDA)~\cite{ye2005characterization}, etc.. These approaches mainly compute the eigendecomposition of matrix $\SS_t^{-1}\SS_b$, but use different formulation of total scatter matrix $\SS_t$~\cite{ye2010discriminant}.

Maximum Margin Criteria (MMC) \cite{li2006efficient} is a simpler and more efficient method. MMC finds a subspace projection matrix $\GG$ to maximize $\Tr(\GG^T(\SS_b-\SS_w)\GG)$. Though in a different way, MMC also maximizes between-class distance while minimizing within-class distance. Semi-Definite Positive LDA (sdpLDA) \cite{kong2012semi} solves the maximization of $\Tr(\GG^T(\SS_b-\lambda_1\SS_w)\GG)$, where $\lambda_1$ is the largest eigenvalue of $\SS_w^{-1}\SS_b$. sdpLDA is derived from the maximum margin principle.

Multi-label problem arise frequently in image and video annotations and many other related applications, such as multi-topic text categorization \cite{wang2010multi}. There are many Multi-label dimension reduction approaches, such as Multi-label Linear Regression (MLR), Multi-label informed Latent Semantic Indexing (MLSI) \cite{yu2005multi}, Multi-label Dimensionality reduction via Dependence Maximization (MDDM) \cite{zhang2010multilabel}, Multi-Label Least Square (MLLS) \cite{ji2008extracting}, Multi-label Linear Discriminant Analysis (MLDA) \cite{wang2010multi}.

\section{Experiments}

\begin{table}[t]
\centering
\small
\caption{Multi-label datasets attributes.}
\label{tab:data2}
\scalebox{1.2}{
\begin{tabular}{c|ccc}
\hline\hline
Data  & n     & p     & k \\
\hline
    MSRC-MOM  & 591   & 384   & 23 \\
    Barcelona & 139   & 48    & 4 \\
    Emotion & 593   & 72    & 6 \\
    Yeast & 2,417 & 103   & 14 \\
    MSRC-SIFT & 591   & 240   & 23 \\
    Scene & 2,407 & 294   & 6 \\
\hline\hline
\end{tabular}}
\end{table}

\begin{table}[t]
\centering
\caption{Classification accuracy on Multi-label datasets ($K-1$ dimension).}
\label{tab:m_acc}
\resizebox{0.7\columnwidth}{!}{
%\scalebox{1}{
\begin{tabular}{c|ccccc}
\hline\hline
Data  & kaLDA    & MLSI  & MDDM  & MLLS  & MLDA  \\
\hline
MSRC-MOM  & \textbf{0.9150}  & 0.8962 & 0.9044 & 0.8994 & 0.9036  \\
Barcelona & \textbf{0.6579}  & 0.6436 & 0.6470 & 0.6524 & 0.6290  \\
Emotion & \textbf{0.7634}  & 0.7397 & 0.7540 & 0.7529 & 0.7619 \\
Yeast & \textbf{0.7405} & 0.7317 & 0.7371 & 0.7364 & 0.7368 \\
MSRC-SIFT & 0.8839 & 0.8762 & 0.8800 & 0.8807 & \textbf{0.8858} \\
Scene & \textbf{0.8870}  & 0.8534 & 0.8713 & 0.8229 & 0.8771  \\
\hline\hline
\end{tabular}}
\end{table}

In this section, we first compare  kernel alignment LDA (kaLDA) with other six different methods on 8 single label data sets
and compare kaLDA multi-label version with four other methods on 6 multi-label data sets.

\subsection{Comparison with Trace Ratio w.r.t. subspace dimension}

\begin{figure*}[!]
\centering
\begin{subfigure}{.5\textwidth}
  \centering
  \includegraphics[width=.9\textwidth]{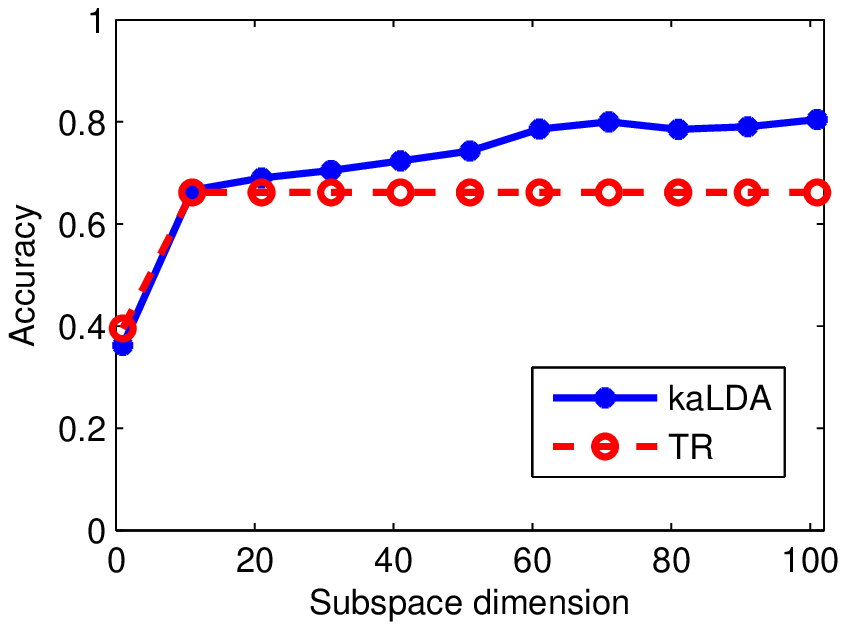}
  \caption{Caltec07}
  \label{fig:1}
\end{subfigure}%
\begin{subfigure}{.5\textwidth}
  \centering
  \includegraphics[width=.9\textwidth]{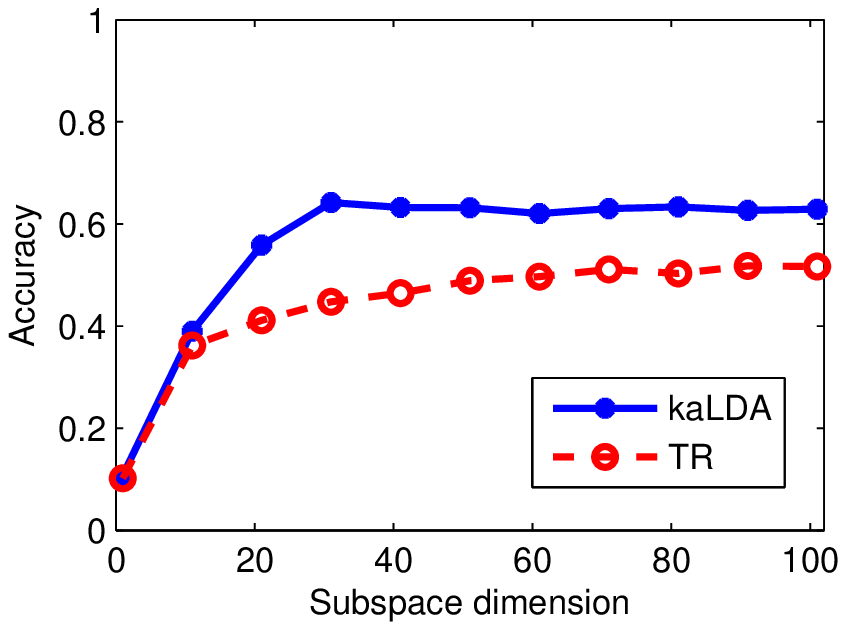}
  \caption{Caltec20}
  \label{fig:2}
\end{subfigure}\\
\begin{subfigure}{.5\textwidth}
  \centering
  \includegraphics[width=.9\textwidth]{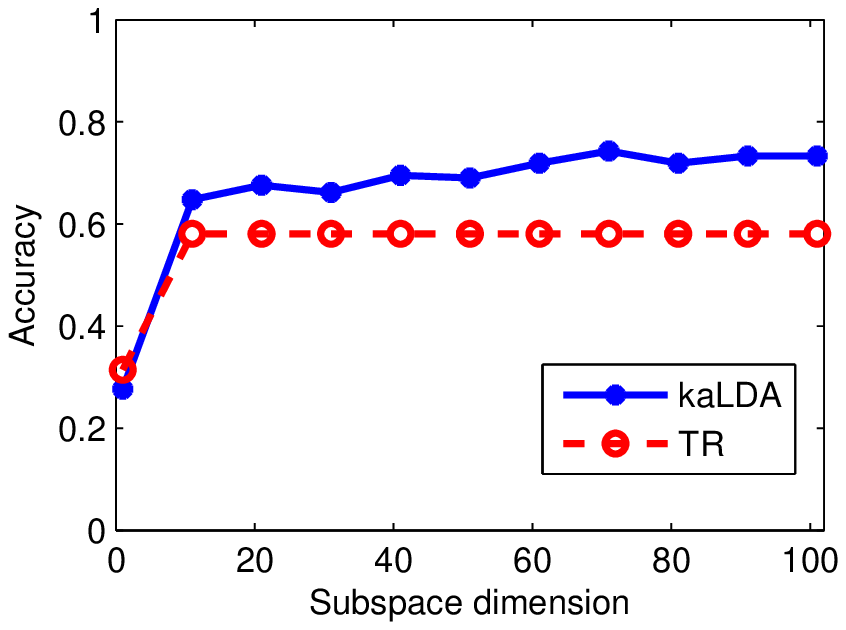}
  \caption{MSRC}
  \label{fig:3}
\end{subfigure}%
\begin{subfigure}{.5\textwidth}
  \centering
  \includegraphics[width=.9\textwidth]{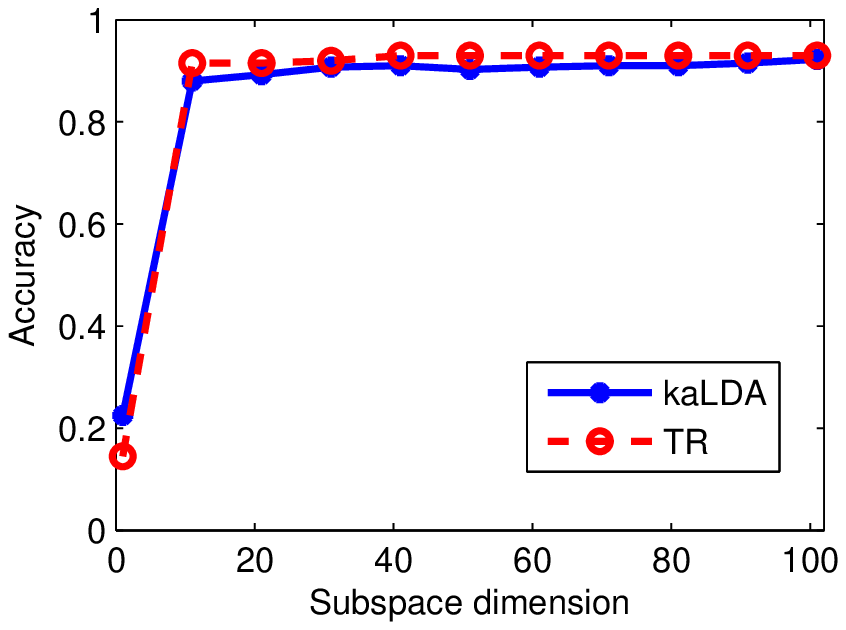}
  \caption{ATT}
  \label{fig:4}
\end{subfigure}\\
\begin{subfigure}{.5\textwidth}
  \centering
  \includegraphics[width=.9\textwidth]{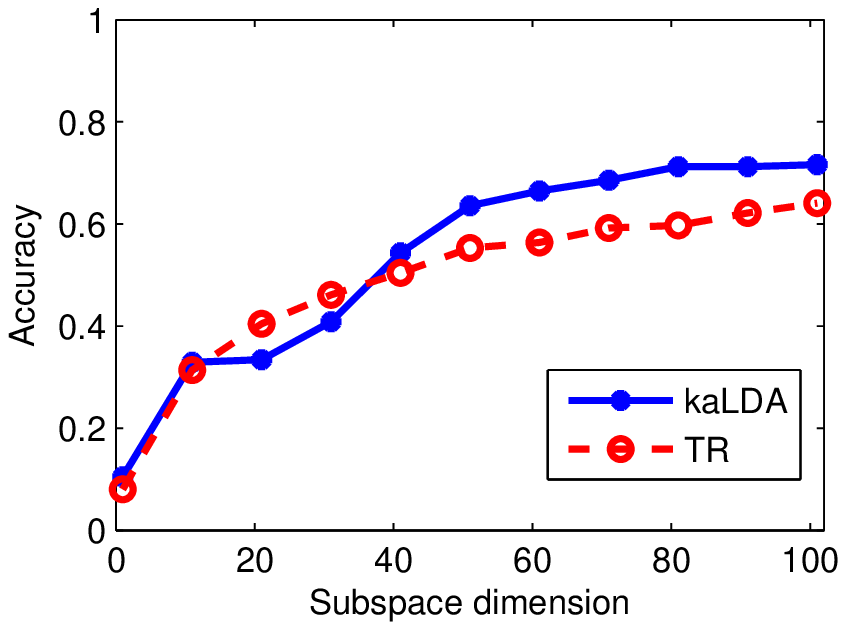}
  \caption{Binalpha}
  \label{fig:5}
\end{subfigure}%
\begin{subfigure}{.5\textwidth}
  \centering
  \includegraphics[width=.9\textwidth]{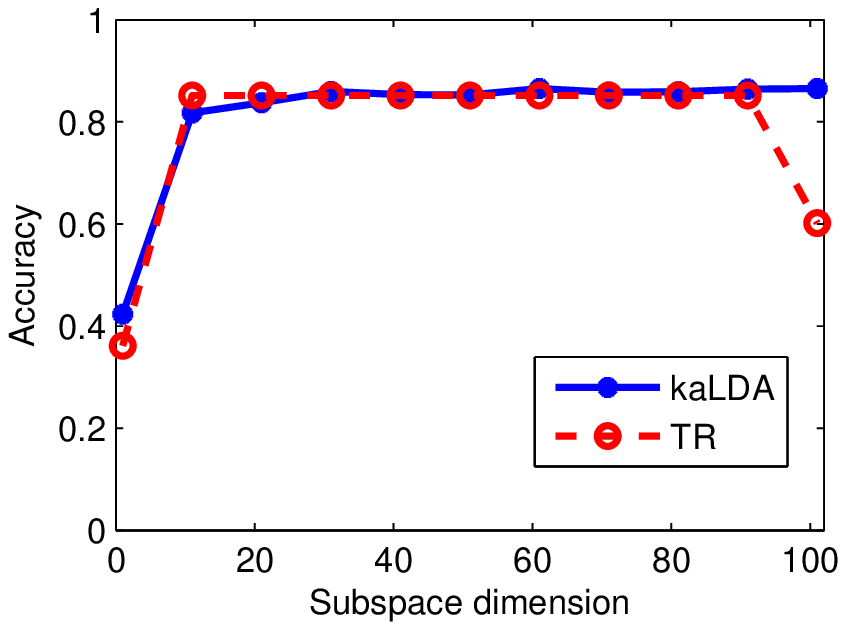}
  \caption{Mnist}
  \label{fig:6}
\end{subfigure}\\
\begin{subfigure}{.5\textwidth}
  \centering
  \includegraphics[width=.9\textwidth]{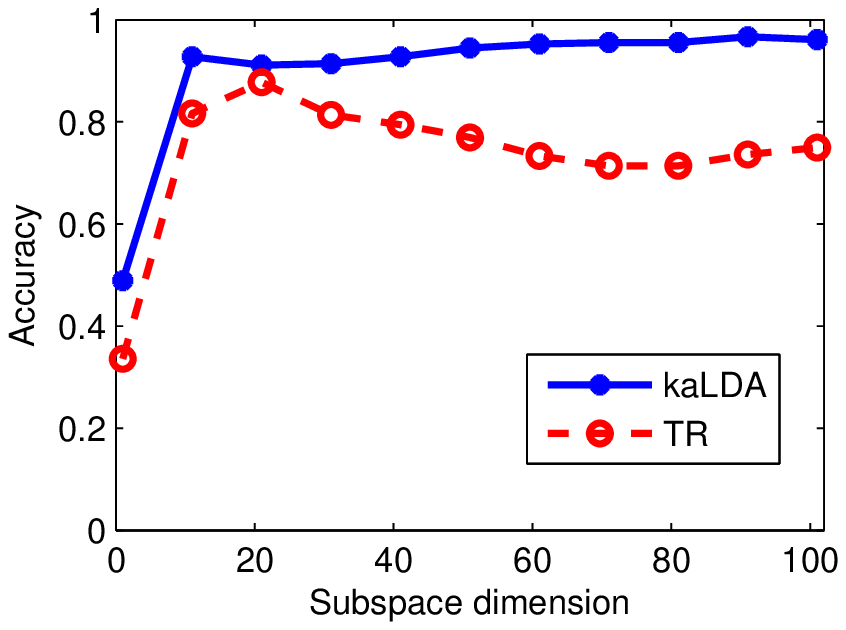}
  \caption{Umist}
  \label{fig:7}
\end{subfigure}%
\begin{subfigure}{.5\textwidth}
  \centering
  \includegraphics[width=.9\textwidth]{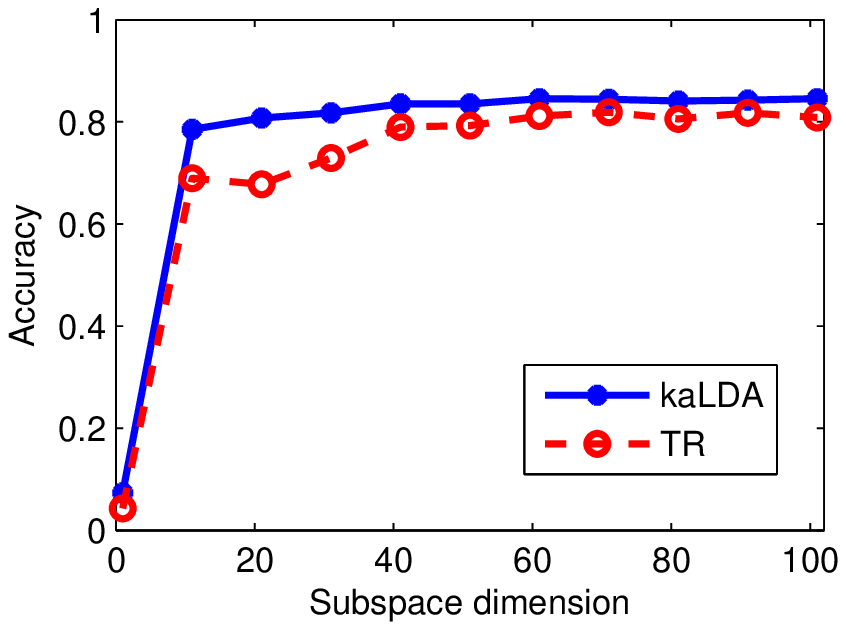}
  \caption{Pie}
  \label{fig:8}
\end{subfigure}
\caption{Classification accuracy w.r.t. dimension of the subspace.}
\label{fig:fn}
\end{figure*}

Eight single-label datasets are used in this experiment. These datasets come from different domains, such as image scene Caltec \cite{caltech} and MSRC \cite{msrc}, face datasets ATT, Umist, Pie \cite{pie}, and digit datasets Mnist \cite{mnist} and Binalpha. Table \ref{tab:data1} summarizes the attributes of those datasets.

{\bf Caltec07} and {\bf Caltec20} are subsets of Caltech 101 data. Only the HOG feature is used in this paper.

{\bf MSRC} is a image scene data, includes tree, building, plane, cow, face, car and so on. It has 210 images from 7 classes and each image has 432 dimension.

{\bf ATT} data contains 400 images of 40 persons, with 10 images for each person. The images has been resized to $28 \times 23$.

{\bf Binalpha} data contains 26 binary hand-written alphabets. It has 1014 images in total and each image has 320 dimension.

{\bf Mnist} is a handwritten digits dataset. The digits have been size-normalized and centred. It has 10 classes and 150 images in total, with 784 dimension each image.

{\bf Umist} is a face image dataset (Sheffield Face database) with 360 images from 20 individuals with mixed race, gender and appearance.

{\bf Pie} is a face database collected by Carnegie Mellon Robotics Institute between October and December 2000. In total, it has 68 different persons.

In this part, we compare the classification accuracy of kaLDA and Trace Ratio \cite{wang2007trace} with respect to subspace dimension.
The dimension of the subspace that kaLDA can find is not restricted to $K-1$.
After subspace projection, KNN classifier ($knn=3$) is applied to perform classification. Results are shown in Figure \ref{fig:fn}. Solid line denotes kaLDA accuracy and dashed line denotes Trace Ratio accuracy. As we can see, in Figures \ref{fig:1}, \ref{fig:2}, \ref{fig:3}, \ref{fig:7}, and \ref{fig:8}, kaLDA has higher accuracy than Trace Ratio when using the same number of reduced features. In Figures \ref{fig:4}, \ref{fig:5}, \ref{fig:6}, kaLDA has competitive classification accuracy with Trace Ratio. However, kaLDA is more stable than Trace Ratio. For example, in Figure \ref{fig:6} and \ref{fig:7}, we observe a decrease in accuracy when feature number increases using Trace Ratio.

\subsection{Comparison with other LDA methods}

\begin{table}[t]
\centering
\caption{Macro F1 score on Multi-label datasets ($K-1$ dimension).}
\label{tab:macrof}
\resizebox{0.7\columnwidth}{!}{
\begin{tabular}{c|ccccc}
\hline\hline
Dataset & kaLDA    & MLSI  & MDDM  & MLLS  & MLDA   \\
\hline
MSRC-MOM  & \textbf{0.6104} & 0.5244 & 0.5593 & 0.5426 & 0.5571 \\
Barcelona & \textbf{0.7377}  & 0.7286 & 0.7301 & 0.7341 & 0.7169  \\
Emotion & \textbf{0.6274} & 0.5873 & 0.6101 & 0.6041 & 0.6200  \\
Yeast & \textbf{0.5757}  & 0.5568 & 0.5696 & 0.5691 & 0.5693  \\
MSRC-SIFT & 0.4712  & 0.4334 & 0.4522 & 0.4544 & \textbf{0.4773} \\
Scene & \textbf{0.6851} & 0.5911 & 0.6411 & 0.5048 & 0.6568  \\
\hline\hline
\end{tabular}}
\end{table}

\begin{table}[t!]
\centering
\caption{Micro F1 score on Multi-label datasets ($K-1$ dimension).}
\label{tab:microf}
\resizebox{0.7\columnwidth}{!}{
\begin{tabular}{c|ccccc}
\hline\hline
Dataset & kaLDA     & MLSI  & MDDM  & MLLS  & MLDA   \\
\hline
MSRC-MOM  & \textbf{0.5138}  & 0.4064 & 0.4432 & 0.4370 & 0.4448  \\
Barcelona & \textbf{0.6969}  & 0.6891 & 0.6861 & 0.6904 & 0.6772  \\
Emotion & \textbf{0.6203}  & 0.5779 & 0.6030 & 0.5961 & 0.6151  \\
Yeast & \textbf{0.4249}  & 0.4026 & 0.4205 & 0.4216 & 0.4213 \\
MSRC-SIFT & 0.3943  & 0.3510 & 0.3637 & 0.3667 & \textbf{0.3959}  \\
Scene & \textbf{0.6966} & 0.6006 & 0.6493 & 0.5062 & 0.6643  \\
\hline\hline
\end{tabular}}
\end{table}

We compare kaLDA with six other different methods, including LDA, Trace Ratio (TR), spdLDA, Maximum Margin Criteria (MMC), regularized LDA (RLDA), and Orthogonal Centroid Method (OCM). All LDA will reduce data to $K-1$ dimension. KNN ($knn=3$) will be applied to do the classification after data is projected into the selected subspace. The other algorithms have already been introduced in related work section. The final classification accuracy is the average of 5-fold cross validation, and is reported in Table \ref{tab:acc}. The first column ``kaLDA'' reports kaLDA classification accuracy. kaLDA has the highest accuracy on 4 out of 8 datasets, including Caltec20, MSRC-MOM, ATT and Mnist. For Umist and Pie, kaLDA results are very close to the highest accuracy. Overall, kaLDA performs better than all other methods.

\subsection{Multi-label Classification}

Six multi-label datasets are used in this part. These datasets include images features, music emotion and so on. Table \ref{tab:data2} summarizes the attributes of those datasets.

{\bf MSRC-MOM} and {\bf MSRC-SIFT} data set is provided by Microsoft Research in Cambridge. It includes 591 images of 23 classes. {\bf MSRC-MOM} is the Moment invariants (MOM) feature of images and each image has 384 dimensions. {\bf MSRC-SIFT} is the SIFT feature and each image has 240 dimensions. About 80\% of the images are annotated with at least one classes and about three classes per image on average.

{\bf Barcelona} data set contains 139 images with 4 classes, i.e., ``building'', ``flora'', ``people'' and ``sky''. Each image has at least two labels.

{\bf Emotion} \cite{trohidis2008multi} is a music emotion data, which comprises 593 songs with 6 emotions. The dimension of Emotion is 72.

{\bf Yeast} \cite{elisseeff2001kernel} is a multi-label data set which contains functional classes of genes in the Yeast Saccharomyces cerevisiae.

{\bf Scene} \cite{boutell2004learning} contains images of still scenes with semantic indexing. It has 2407 images from 6 classes.

We use 5-fold cross validation to evaluate classification performance of different algorithms. K-Nearest Neighbour (KNN) classifier is used after the subspace projection. The algorithms we compared in this section includes Multi-label informed Latent Semantic Indexing (MLSI), Multi-label Dimensionality reduction via Dependence Maximization (MDDM), Multi-Label Least Square (MLLS), Multi-label Linear Discriminant Analysis (MLDA). These algorithms have been introduced in related work section.

We compare the performance of kaLDA and other algorithms using macro accuracy (Table \ref{tab:m_acc}), macro-averaged F1-score (Table \ref{tab:macrof}) and micro-averaged (Table \ref{tab:microf}) F1-score. Accuracy and F1 score are computed using standard binary classification definitions. In multi-label classification, macro average is a standard class-wise average, and it is related to number of samples in each class. However, micro average gives equal weight to all classes \cite{wang2010multi}. kaLDA achieves highest classification accuracy on 5 out of 6 datasets.
On the remaining  MSRC-SIFT dataset, kaLDA result is very close to the best method MLDA and beat all rest methods. kaLDA achieves highest macro and micro F1 score on 5 out of 6 datasets. Furthermore, kaLDA has the second highest macro and micro F1 score on dataset MSRC-SIFT. Overall, kaLDA outperforms other multi-label algorithms in terms of classification accuracy and macro and micro F1 score.

\section{Conclusions}
In this paper, we propose a new kernel alignment induced LDA  (kaLDA). The objective function of kaLDA is very similar to classical LDA objective. The Stifel-manifold  gradient descent algorithm can solve kaLDA objective efficiently. We have also extended kaLDA to multi-label problems. Extensive experiments show the effectiveness of kaLDA in both single-label and multi-label problems.

\noindent {\bf Acknowledgment}.
This work is partially supported  by US NSF CCF-0917274 and NSF DMS-0915228.

\bibliographystyle{splncs03}
\bibliography{bib1}

\end{document}